\documentclass[letterpaper]{article} 
\usepackage{aaai24} 
\usepackage{times}  
\usepackage{helvet} 
\usepackage{courier}  
\usepackage[hyphens]{url}  
\usepackage{graphicx} 
\usepackage{natbib}  
\usepackage{caption} 
\usepackage{subfigure} 
\usepackage{booktabs} 
\usepackage{amsmath} 
\usepackage{amsthm} 
\newtheorem{theorem}{Theorem} %
\usepackage{cuted} 
\usepackage{todonotes}
\usepackage{algorithm}
\usepackage{algorithmic}

\usepackage{newfloat}
\usepackage{listings}
\DeclareCaptionStyle{ruled}{labelfont=normalfont,labelsep=colon,strut=off} 
\floatstyle{ruled}
\newfloat{listing}{tb}{lst}{}
\floatname{listing}{Listing}

\title{Towards Fairness in Online Service with $k$ Servers\\ and its Application on Fair Food Delivery}
\author {
    Daman Deep Singh,
    Amit Kumar,
    Abhijnan Chakraborty
}
\affiliations {
    Department of Computer Science and Engineering\\
    Indian Institute of Technology Delhi, India
}

\newcommand{\cI}{{\mathcal I}}

\newcommand{\mpart}{\textsc{m-Partition}}
\newcommand{\kfood}{\textsc{$k$-food}}
\newcommand{\fairkfood}{\textsc{Fair $k$-food}}
\newcommand{\pageTW}{\textsc{PageTW}}
\newcommand{\kserver}{\textsc{$k$-server}}
\newcommand{\kserverTW}{\textsc{$k$-serverTW}}

\newcommand{\ktaxi}{\textsc{$k$-taxi}}

\newcommand{\fairKtaxi}{\textsc{Fair $k$-taxi}}

\newcommand{\workfood}{\textsc{Work4Food}}
\newcommand{\flowlp}{\textsc{FlowMILP}} 
\newcommand{\flpsf}{\textsc{FlowMILP(2S)}}

\newcommand{\random}{\textsc{Random}}
\newcommand{\roundRobin}{\textsc{RoundRobin}}
\newcommand{\greedyMin}{\textsc{GreedyMin}}
\newcommand{\minDelta}{\textsc{MinDelta}}
\newcommand{\docfood}{\textsc{Doc4Food}}
\newcommand{\dc}{\textsc{Double Coverage}}

\newcommand{\foodmatch}{\textsc{FoodMatch}}
\newcommand{\fairfoody}{\textsc{FairFoody}}

\newcommand{\sparse}{\textsc{SynSparse}}
\newcommand{\dense}{\textsc{SynDense}}

\newcommand{\dset}{{\mathcal D}}
\newcommand{\mset}{{\mathcal M}}

\newcommand{\rset}{{\mathcal R}}

\newcommand{\gset}{{\mathcal G}}
\newcommand{\xset}{{\mathcal X}}
\newcommand{\bigO}{{\mathcal O}}

\newcounter{note}[section]

\begin{document}

\maketitle

\begin{abstract}
The \kserver\ problem is one of the most prominent problems in 
online algorithms with several variants and extensions. 
However, simplifying assumptions like instantaneous server movements and zero service time has hitherto limited its applicability to real-world problems. In this paper, 
we introduce a realistic generalization of \kserver\ without such assumptions -- the \kfood\ problem, where requests with source-destination locations and an associated pickup time window arrive in an online fashion, and each has to be served by exactly one of the available $k$ servers. The \kfood\ problem offers the versatility to model a variety of real-world use cases such as food delivery, ride sharing, and 
quick commerce. Moreover, motivated by the need for fairness in online platforms, we introduce the \fairkfood\ problem with the \textit{max-min} objective. We establish that both \kfood\ and \fairkfood\ problems are strongly NP-hard and develop an optimal offline algorithm that arises naturally from a time-expanded flow network. Subsequently, we propose an online algorithm \docfood\ involving virtual movements of servers to the nearest request location. Experiments on a real-world food-delivery dataset, alongside synthetic datasets, establish the efficacy of the proposed algorithm against state-of-the-art fair food delivery algorithms.
\end{abstract}

\section{Introduction}

The \kserver\ problem~\cite{kserver} 
is one of the most studied problems in the domain of online algorithms. In this problem, a 
sequence of requests arrives online at various locations in a $m$-point metric space and each request has to be served by one of the $k$ servers by moving the server to the corresponding location, and the objective is to minimize the total movement of the servers. Owing to its significance, a number of variants of this problem have been explored in the past. For instance, the \ktaxi\ problem~\cite{ktaxi} extends the \kserver\ problem to consider each request as a pair of points in the metric space. A server must move from a request's source point to the corresponding destination point in order to fulfill the request. The goal is to efficiently assign the $k$ taxis to minimize the total travel distance. Another notable variant 
is the $k$-server with Time Windows (\kserverTW) problem~\cite{kserverTW} where each request, additionally, has a deadline associated with it 
within which it needs to be served, allowing a server to handle several `live' requests in a single visit. Many of these extensions have the capacity to model specific 
problems like caching, path planning, and resource allocation.

However, all existing variants of the \kserver\ problem assume that the server movement is instantaneous, i.e., once a server is assigned to a particular request, it takes no time for it to move to the request location. Moreover, there is no service time associated with a request, and thus all $k$ servers are available to serve a request at time $t$ even if some of them were assigned a request at time $t-1$. Such simplifying assumptions limit the applicability of \kserver\ problem 
to more realistic scenarios.

To overcome these issues, in this work, we introduce a general problem, called the \kfood\ problem, which builds upon a number of aforementioned $k$-server extensions but is more rooted in reality. In this problem, each request corresponds to a pair of points -- source and destination -- in a metric space accompanied by a pick-up (or preparation) time window. Serving a request involves moving one of the $k$ servers to its source location within the pick-up time window and subsequently moving to its destination. Importantly, the servers take finite amount of time to travel, during which they are unavailable to serve a new request. The objective of \kfood\ is still to minimize the net server movement. 


Going further, 
motivated by the recent reports highlighting the difficult condition of gig delivery drivers in the global south~\cite{fairfoody,work4food, fair-ride-sharing-1, fair-ride-sharing-2, fairassign}, particularly their struggle to earn even minimum wage, 
we also introduce a variant of \kfood\ problem, called the \fairkfood\ problem which assumes a \textit{max-min} objective instead of the \textit{min} cost objective. 
This objective, inspired by Rawls' theory of justice~\cite{rawls1971theory}, aims to maximize the minimum reward earned by any server. We demonstrate the applicability of \fairkfood\ problem in ensuring fairness in food delivery platforms.


Today, platforms like DoorDash, Deliveroo and Zomato have become de facto destinations for ordering food. Apart from serving many customers, they also provide livelihood to millions of delivery drivers worldwide. 
Although multiple approaches have been proposed to ensure fair driver assignment~\cite{fairfoody, work4food}, they all adopt a \textit{semi-online} approach where they collect requests within an accumulation time window and then apply an offline algorithm to match with eligible drivers. The underlying problem, however, is inherently online, where food orders (requests) arrive one by one and have to be assigned to one of the eligible drivers (servers). In this work, apart from developing an offline solution scalable up to thousands of requests and hundreds of servers, we propose the first \textit{purely online} driver assignment algorithm for food delivery, which we call \docfood. Extensive experiments on synthetic and real food-delivery data establish the superiority 
of \docfood\ compared to the semi-online solutions concerning the fairness objective. 

\vspace{1mm} \noindent
\textbf{Our Contributions.} In summary, in this paper, we
\begin{itemize}
    \item introduce the \kfood\ and \fairkfood\ problems with the potential to model multiple real-world applications, and show that both problems are strongly NP-hard; 
    \item design a fractional offline-optimal algorithm for the \fairkfood\ problem 
    utilizing the corresponding time-expanded flow network; 
    \item propose an online algorithm \docfood\ for fair food delivery, employing a prominent heuristic in online algorithms informed by domain-specific knowledge; and 
    \item present extensive experimental analysis on a real-world food delivery dataset and two synthetic datasets.
\end{itemize}
\section{Related Works}
\textbf{\kserver\ and its variants.} The online \kserver\ problem~\cite{kserver} is arguably the most prominent problem in online algorithms. 
Over the past few decades, numerous variations of this problem have been explored, including paging~\cite{paging}, $k$-sever with time windows~\cite{kserverTW}, delayed $k$-server~\cite{delayed-kserver}, $k$-server with rejection~\cite{kserver-penalty}, online $k$-taxi~\cite{ktaxi}, Stochastic $k$-server~\cite{stochasticKserver} and, $k$-server with preferences~\cite{kserver-pref}. 
However, in all these variants, server movement is always instantaneous. Our proposed \kfood\ problem moves beyond this assumption and captures the subtleties of real-world settings.

\vspace{1mm} \noindent
\textbf{Fairness in \kserver.} The \kserver\ problem and its variants have traditionally been studied with the objective of minimizing the total movement cost. To 
our knowledge, the only other work that presents a fairness-motivated objective is~\cite{minmaxPaging}, where the paging problem -- a special case of the \kserver\ problem -- is studied with the {\it min-max} objective. They design a deterministic $\bigO (k \log(n)\log(k))$-competitive algorithm and an $O(\log^2 n \log k)$-competitive randomized algorithm for the online min-max paging problem. They also showed that any deterministic algorithm for this problem has a competitive ratio $\Omega(k \log n/ \log k)$ and any randomized algorithm has a competitive ratio $\Omega(\log n)$. 

\vspace{1mm} \noindent
\textbf{Fairness in online platforms.} The growing prevalence of online platforms in various domains, including ride hailing, food delivery and e-commerce, has attracted increasing attention to these research areas~\cite{gupta2023towards,chakraborty2017fair,foodmatch,fair-ride-sharing-1}. 
For instance, research on ride hailing has focused on efficiency maximization~\cite{ride1, ride2} and more recently on promoting fairness~\cite{fair-ride-sharing-1, fair-ride-sharing-2}. Online vehicle routing problem~\cite{olvrp} is an interesting work that bears resemblance to our \kfood\ problem. They developed an efficient offline mixed-integer optimization framework that scales well to real-world 
workload via sparsification and re-optimization of the offline optimal. 

Similarly, research on food delivery has seen similar shift from increasing efficiency by minimizing travel costs~\cite{foodmatch, food1, food3} 
to developing equitable food delivery algorithms~\cite{fairfoody,work4food,fairassign}. 
Yet, limited exploration exists regarding purely online solutions tailored to fair food delivery. We attempt to fill this gap in the current work.

\begin{figure*}[t]
    \centering
    \subfigure[]{\includegraphics[height=0.24\textwidth,width=0.30\textwidth]{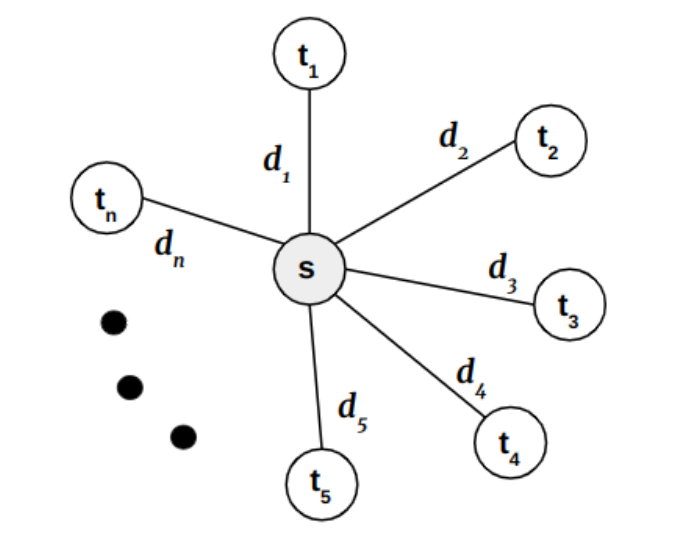} 
    \label{fig:syn}}
    \qquad
    \subfigure[]{\includegraphics[height=0.25\textwidth,width=0.60\textwidth]{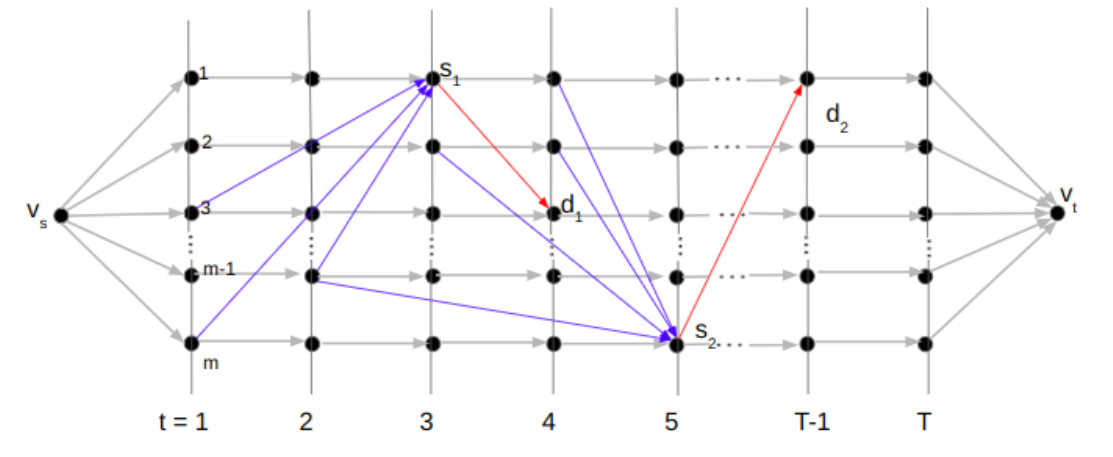} 
    \label{fig:flow-net}}
    \vspace*{-3mm}
    \caption{(a) Star metric with $n+1$ nodes, and (b) Time-Expanded (or Time-Indexed) Flow Network.}
    \vspace*{-4mm}
\end{figure*}

\section{Problem Statement}
\label{sec:model}
Next, we formally describe the classical \kserver\ problem and one of its extensions -- \kserverTW\ problem, and then 
introduce the \kfood\ and \fairkfood\ problems.

\vspace{1mm} \noindent
\textbf{Definition 1.} (The \kserver\ problem)
Consider an $m$-point metric space $(\mset, d)$, an online sequence of requests $\sigma = \{r_1, r_2, \cdots, r_n\}$, and a set of $k$ servers existing at specific, not necessarily distinct, points in the metric space. Each request $r_i$ arrives at a specific point in the metric space and must be served by one of the $k$ servers by moving the server to the corresponding location. The movement of a server incurs a cost equivalent to the distance between the current server location and the requested location. The objective is to minimize the total movement cost. The server movement between any two points 
is assumed to be instantaneous. 
Hence, whenever a new request $r_i$ arrives, all $k$ servers are immediately available for assignment.

\vspace{1mm} \noindent
\textbf{Definition 2.} (The \kserverTW\ problem)
The $k$-server with Time Windows (\kserverTW) problem extends the \kserver\ problem to accommodate an additional \textit{deadline} associated with a request. Specifically, each request $r$ arriving at a point in the metric space at some time $t_{r}^{b}$ with a deadline $t_{r}^{e} \geq t_{r}^{b}$ must be served by moving one of the $k$ servers to the corresponding location within the time window $[t_{r}^{b},t_{r}^{e}]$. The non-triviality (beyond \kserver) of this problem lies in the fact that several requests at the same location can be served by a single server visit 
to this location.

\vspace{1mm} \noindent
\textbf{Problem 1.} (The \kfood\ problem) 
Consider a metric space $(\mset,d)$ comprising $m$ points. 
$k$-servers exist at specific points of $\mset$, constituting the initial server configuration. Requests from a predefined sequence $\sigma = \{r_1, r_2, \cdots, r_n\}$ arrive one-by-one at specific locations in $\mset$. 
Each request $r_{j}$ is a $4$-tuple $(s_{j},d_{j},t_{j}^{b}, t_{j}^{e})$. Here $s_{j}, t_{j}$ are points in $\mset$ and $[t_{j}^{b}, t_{j}^{e}]$ is the time-window associated with the request, also known as the \textit{pick-up} (or \textit{preparation}) time window. Specifically, request $r_{j}$ arrives at its source $s_{j}$ at time $t_{j}^{b}$ and is considered served if one of the $k$-servers can reach $s_{j}$ before the deadline $t_{j}^{e}$ and subsequently move to the destination $d_{j}$. Note that the source-destination travel for each request is fixed and not subject to any deadline.


Unlike traditional \kserver\ setting, the \kfood\ problem considers travel time for server movements. When serving a request, the corresponding server becomes temporarily unavailable for other requests and receives a reward equal to the distance it moves. The primary objective of the \kfood\ problem is to minimize the net reward earned by all servers.

\vspace{1mm} \noindent
\textbf{Problem 2.} (The \fairkfood\ problem)
A variant of the \kfood\ problem with a \textit{maxmin} objective: maximizing the minimum reward earned by any server. This fairness objective has its root in John Rawls's theory of justice~\cite{rawls1971theory} which advocates for ensuring that worst-off people do as well as possible. 

Note that for both \kfood\ problem and \fairkfood\ problem, all server movements are on shortest paths. 



\section{Hardness Results}
Next, we prove that both \kfood\ and \fairkfood\ problems are strongly NP-hard by reductions from the \pageTW\ problem and the \mpart\ problem respectively, which 
are known to be strongly NP-hard.

\vspace{1mm} \noindent
\textbf{Definition 3.} (The \pageTW\ problem) Given a computing process working on $n$ pages of data, with access to two memory levels: a fast cache capable of holding $k < n$ pages, and a slower memory (e.g., disk) containing all $n$ pages. Initially, all pages reside in the slower memory. Each page $p$ carries a weight $w_p$. When the process accesses (or requests) a page, it's either fetched from the cache or prompts a page fault, requiring it to enter the cache and potentially evicting an existing page -- termed as \textit{serving} a page request. Each page request comes with a deadline by which it must be served. The objective is to minimize the total weight of evicted pages while satisfying the specified deadlines.

\begin{theorem}
    The \kfood\ problem is NP-hard.
\end{theorem}
\begin{proof}
    The {\em unit-weight} \pageTW\ problem, where all pages $p$ have $w_p = 1$, was shown to be NP-hard by~\cite{pagetw}.  We give a polynomial time reduction from the unit weight \pageTW\ problem to the \kfood\ problem. Consider such an instance $\cI$ of the \pageTW\ problem specified by a set of $n$ pages, cache size $k$, and a request sequence of length $T$. The request sequence at time $t$ is specified by a page $p_t$ and deadline $D_t$. We map the instance $\cI$ to a \kfood\ instance $\cI'$ as follows. The metric space in $\cI'$ is given by uniform metric on $m$ points labeled $\{1, \ldots, m\}$ and there are $k$ servers, each with infinite speed. The request sequence has length $T$, where the request $r_t$ at time $t$ is given by the 4-tuple $(p_t, p_t, t, D_t)$, i.e., a server needs to visit point $p_t$ during $[t,D_t]$ (we assume w.l.o.g. that the pages are also numbered from $1$ to $n$). We show that the instance $\cI$ has a solution with eviction cost $C$ iff the instance $\cI'$ has a solution of total movement cost $C$. Suppose $\cI$ has a solution $\cal S$ of cost $C$. We produce a solution ${\cal S}'$ for $\cI'$ as follows. At every time $t$, if the cache in the solution $\cal S$ contains pages $i_1, \ldots, i_k$, then the  $k$ servers in the solution ${\cal S}'$ are present at $i_1, \ldots, i_k$ at time $t$. Suppose at a time $t$, a page $p$ gets evicted and a page $p'$ is brought in the cache in the solution $\cal S$. Then at this time, the server from location $p$ moves to location $p'$ in ${\cal S}'$. Thus, both solutions incur the same cost. Similarly, we can show that if there is a solution of cost $C$ to $\cI'$, then there is a solution of cost $C$ to $\cI$.


\end{proof}

\noindent
\textbf{Definition 4.} ~(Multi-way number partitioning: \mpart). Given a multi-set of $n$ positive integers $\dset = \{d_1, d_2, \cdots, d_n\}$ with sum $\Sigma_{d_j \in \dset} d_j = kT$ and a target sum $T$, does there exist a partition $\{D_1, D_2, \cdots, D_k\}$ of $\dset$ such that for each partition $D_h$, $\Sigma_{d_j \in D_h} d_j = T$?

The multi-way number partitioning problem is a well-known strongly NP-hard problem~\cite{npc-results}. Leveraging this result, we'll show that the \fairkfood\ is strongly NP-hard as well.

\begin{theorem}
\fairkfood\ problem is strongly NP-hard.
\end{theorem}
\begin{proof}
    We start by defining a special case of the \fairkfood\ problem with $t_{b}^{r}=t_{e}^{r}$ for all requests $r$ and negligible service time. We call this problem the \fairKtaxi\ problem as it resembles the \ktaxi\ problem~\cite{ktaxi} except the objective. We first show that the \fairKtaxi\ problem is strongly NP-hard using a reduction from the \mpart\ problem. Subsequently, we extend this result to the \fairkfood\ problem.
    Hence, consider an instance $\cI$ of the \mpart \ problem as described in~Definition 3. We construct an instance $\cI'$ of the \fairKtaxi\ problem as follows.


    Let us consider a star metric space containing $n+1$ points, with the central node designated as $s$, and the remaining nodes represented as ${t_1, t_2, \cdots, t_n}$ (refer Figure~\ref{fig:syn}). Each node $t_j$ is at a distance $d_j$ from $s$. Initially, all $k$ servers are placed at the node $s$. Now, consider a sequence of $n$ requests denoted by $\sigma = \{r_1, r_2, \cdots, r_n\}$, where each request $r_j$ arrives at source node $s$ and has its destination at $d_j$. 
    In order to serve a request $r_j$, one of the servers must move from $s$ to $t_j$ and then return to $s$ \footnote{The hardness of this restricted problem with the return requirement naturally extends directly to the more general case without it.}.
    By doing so, the server earns a reward equal to its total movement i.e., $2d_j$. Consequently, the set of rewards, say $\dset'$, associated with $\sigma$ becomes $\{2d_1, 2d_2, \cdots, 2d_n\}$. We now claim that the instance $\cI$ has a solution iff instance $\cI'$ has a solution where the minimum reward earned by any server is $2T$.
    
   Indeed, suppose $\cI$ has a solution where the partition of the set $\dset$ is $D_1, \ldots, D_k$. Then, in the instance $\cI'$, we define a solution where server $h$ serves the requests corresponding to $D_h$. It follows that each of the servers collects a reward of $2T$. Conversely, suppose there is a solution to $\cI'$ collects a reward of at least $2T$. Since the total of all rewards is equal to $2Tk$, it follows that each server receives reward of exactly $2T$. Now, construct a solution $D_1, \ldots, D_k$ to the instance $\cI$ where $D_h$ is the set subset of $\dset$ corresponding to the requests served by server $h$. It follows that the total sum of the values in $D_h$ is equal to half the total reward of server $h$, which is equal to $T$. Hence, we have shown that the problem $\fairKtaxi$ \ is NP-hard. 
    Given that the \fairkfood\ problem is a generalization of \fairKtaxi, it follows that the \fairkfood\ problem is also strongly NP-hard. 
    
\end{proof}


\section{Methodology}
\textbf{$\flowlp$: Fractional offline optimal for \fairkfood.} 
\label{subsec:offline-optimal}
We propose a fractional offline solution (\flowlp) for the \fairkfood\ problem. It is a 
Mixed Integer LP (MILP) that, intuitively, follows from the min-cost LP formulation for \kserver. This 
MILP may route fractions of servers. Consequently, serving a request entails moving a unit amount of server towards the request's location. 

Consider an instance of the \fairkfood\ problem with a metric space $(\xset, d)$ on $m$ points and a sequence of $n$ requests $\sigma = \{r_1, r_2, \cdots, r_n\}$. We observe the entire duration $T$, from the arrival of the first order to the end of service of the last order, in timesteps of size $\eta$ so chosen that each request has a distinct~\footnote{This is without loss of generality. We can arrange simultaneous requests arbitrarily and treat them as distinct arrivals.} arrival timestep. 
We now construct a time-expanded graph $\gset (V,E)$ where $V$ is the set of nodes obtained by copying the nodes in $\xset$ at each timestep i.e., $V = \{v_{i,j} | i \in [m], j \in \{0, \eta, 2\eta, \cdots, T\}\} \cup (v_s,v_t)$ where $v_{i,j}$ is the copy of $v_i\in \xset$ at timestep $t$ and $E$ is the set of edges. Here $v_s$ and $v_t$ are special source and sink nodes. ) 
We now construct a time-expanded graph $\gset (V,E)$ where $V$ is the set of nodes obtained by copying the nodes in $\xset$ at each timestep i.e., $V = \{v_{i,j} | i \in [m], j \in \{0, \eta, 2\eta, \cdots, T\}\} \cup (v_s,v_t)$ where $v_{i,j}$ is the copy of $v_i\in \xset$ at timestep $t$ and $E$ is the set of edges. Here $v_s$ and $v_t$ are special source and sink nodes. 

There exist three types of edges in $E$: $(i)$ \textit{Source-Sink edges:} $0$-cost edges connecting source node $v_s$ to the nodes in $\{v_{i,0} | i \in [m]\}$ and terminal node $v_t$ to $\{v_{i,T} | i \in [m]\}$. $(ii)$ \textit{Self edges:} $0$-cost edges between the nodes $v_{i,j}$ and $v_{i,j+1}, \  \forall i \in [m],  j\in\{0,\ldots,(T-1) \}$, indicating that a server may stay at a location.  $(iii)$ \textit{Cross edges:} For every request $r_j=(s_j, d_j, t_{j}^{b}, t_{j}^{e})$, we add the following edges: let $s_{j,t}$ and $d_{j,t}$ denote  the copy of $s_j$ and $d_j$ in $V$ at timestep $t$ respectively. First, we add edges from $v_{h,t'}\in V$ to $s_{j,t_{j}^{e}}$; $t_j^b \leq t' < t_j^e$, $v_h \ne s_j$ if the time taken to traverse on the shortest path from $v_{h,t'}$ to  $s_{j,t_{j}^{b}}$ is at most  $(t_{j}^{e}-t_{j}^{b})$. The cost of this edge is equal to the corresponding shortest path distance. Next, we add an edge between $s_{j,t_{j}^{e}}$ and $d_{j,t}$, where $(t-t_j^e)$ is the time taken to travel from $s_j$ to $d_j$ on the shortest path and the cost of this edge is equal to the length of this path. Note that we can prune some edges here. If we have edges from $v_{h,t_1}$ and $v_{h,t_2}$ to $s_{j,t_j^e}$ for some vertex $v$ and times $t_1 < t_2$, then we can remove the first edge. Indeed, there is no loss in generality in assuming that the server arrives at the source location at time $t_j^e$. This ensures that the server is moved only when the request becomes critical. 
The flow on each edge $e$, denoted as $f_{e}$, comprises of flow from each of the $k$ servers i.e., $f_{e} = \Sigma_{i=1}^{k} f_{e}^{i}$. 
Figure~\ref{fig:flow-net} shows an example flow network with $2$ requests $r_1=(s_1,d_1,1,3)$ and $r_2=(s_2,d_2,2,5)$. 

The \flowlp\ on instance $(\gset(V,E), \sigma)$ is defined as
\begin{align}
    \textbf{max.}~&\mset - p.\sum\limits_{r=1}^{n} z_r \label{eqn:obj}\\
    \textbf{s.t.}~& m_i = \sum\limits_{e \in E} f_{e}^{i}.c_{e}, \forall i \in [k] \label{eqn:2}\\
    & \mset \leq m_i, \forall i \in [k] \label{eqn:3}\\ 
    & \sum\limits_{u\in\delta^{-}(v)}\sum\limits_{i=1}^{k} f^{i}_{(u,v)} \leq 1, \forall v \in \{s_j: j \in [n]\} \label{eqn:4}\\ 
    & z_{r_j} + \sum\limits_{i=1}^{k} f_{(s_j,d_j)}^{i} = 1, \forall r_j \in \sigma \label{eqn:5}\\ 
    & \sum\limits_{v \in \delta^{+}(u)} f^{i}_{(u,v)}=\sum\limits_{v \in \delta^{-}(u)} f^{i}_{(u,v)}, u \in V, i \in [k] \label{eqn:6}\\ 
    & \sum\limits_{i=1}^{m}\sum\limits_{j=1}^{k} f^{j}_{(s,v_{i1})} = \sum\limits_{i=1}^m\sum\limits_{j=1}^{k} f^{j}_{(v_{iT},t)} = k \label{eqn:7}\\ 
    \textbf{vars.} ~& f_{e}^{i} \in [0,1], ~\quad \forall i \in [k], \quad 
                     z_r \in \{0,1\}, \quad \forall r \in \sigma \notag
\end{align}


The objective (\ref{eqn:obj}) captures the goal of maximizing  the minimum reward $\mset$ while minimizing the infeasibilities ($z_r$'s); $p$ being the infeasibility penalty. The binary variable $z_r$ is $1$ iff  request $r$ cannot be served. The variable $f_e^i$ denotes the flow of server $i$ on edge $e$. The minimum reward $\mset$ is computed using the constraints (\ref{eqn:2}) and (\ref{eqn:3}). Constraints (\ref{eqn:4}) and (\ref{eqn:5}) capture that each request $r$ is served by at most $1$ server, and if it is not served, then $z_r$ is 1.  
Constraints (\ref{eqn:6}) is the flow-conservation constraint, whereas constraint (\ref{eqn:7}) refers to the fact that we have $k$ servers.

Overall, the \flowlp\ formulation has $\bigO(mk(n+T))$ decision variables and $\bigO(nkT)$ constraints. The \flowlp\ formulation is flexible and can be easily modified to accommodate various other objectives. For example, the \kfood\ problem can be modeled using \flowlp\ by changing the objective (\ref{eqn:obj}) to a min-cost objective and disregarding the separate flows for each server.


\begin{table*}[t]
    \vspace{-2mm}
    \parbox{0.45\linewidth}{
        \centering
        \caption{Comparison of \flowlp\ against various online algorithms on the \sparse\ dataset.}
        \label{tab:syn-sparse}
            \begin{tabular}{cccc}
            \toprule
             & \textbf{\#Unserved} & \textbf{Cost} & \textbf{Min.R}\\
            \midrule
            \flowlp & $2$ & $2444.50$ & $2444.50$ \\
            \flowlp\textsc{(2S)} & $2$ & $1399.56$ & $1399.56$ \\
            \midrule
            \random & $8$ & $1712.04$ & $723$ \\
            \greedyMin & $10$ & $1703.43$ & $1065$ \\
            \docfood & $5$ & $1766.01$ & $1186$ \\
            \minDelta & $5$ & $1658.12$ & $430$  \\
            \roundRobin & $30$ & $1500.85$ & $10$\\
            \bottomrule
            \end{tabular}
            }
        \hfill 
        \parbox{0.45\linewidth}{
            \centering
            \caption{Comparison of \flowlp\ against various online algorithms on the \dense\ dataset.}
            \label{tab:syn-dense}
            \begin{tabular}{cccc}
            \toprule
             & \textbf{\#Unserved} & \textbf{Cost} & \textbf{Min.R}\\
            \midrule
            \flowlp & $1$ & $7137.44$ & $7137.44$ \\
            \flowlp\textsc{(2S)} & $1$ & $1282.93$ & $1282.93$ \\
            \midrule
            \random & $8$ & $1534.80$ & $800$ \\
            \greedyMin & $8$ & $1550.81$ & $1034$ \\
            \docfood & $7$ & $1554.02$ & $1045$ \\
            \minDelta & $6$ & $1496.73$ & $855$ \\
            \roundRobin & $12$ & $1494.92$ & $10$ \\
            \bottomrule
            \end{tabular}
        }
\vspace*{-2mm}
\end{table*}

\begin{figure*}[t]
  \vspace*{-4mm}
  \centering
  \subfigure[]{\includegraphics[width=0.32\textwidth]{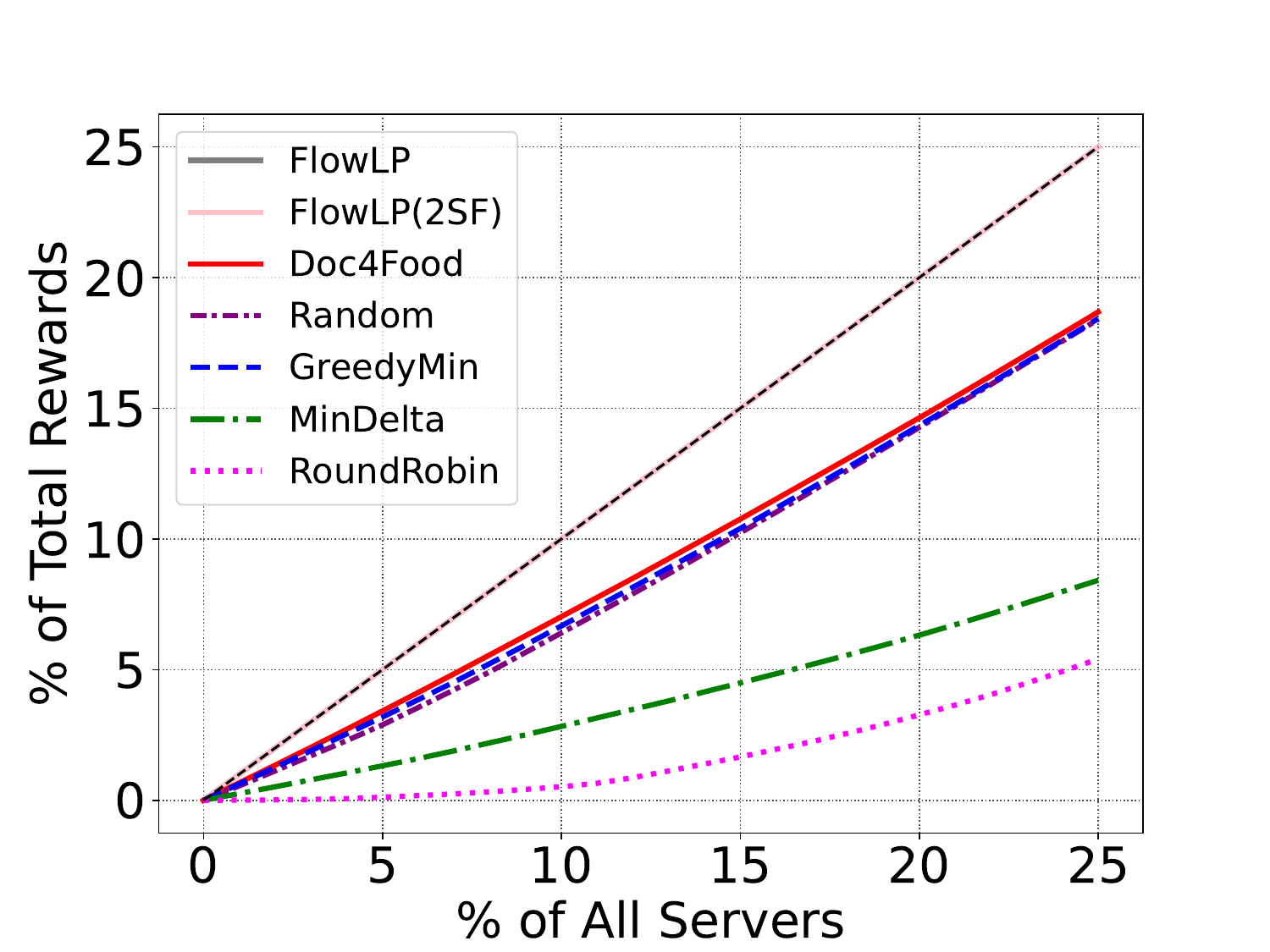} \label{fig:syn-sparse}} 
  \hfill
  \subfigure[]{\includegraphics[width=0.32\textwidth]{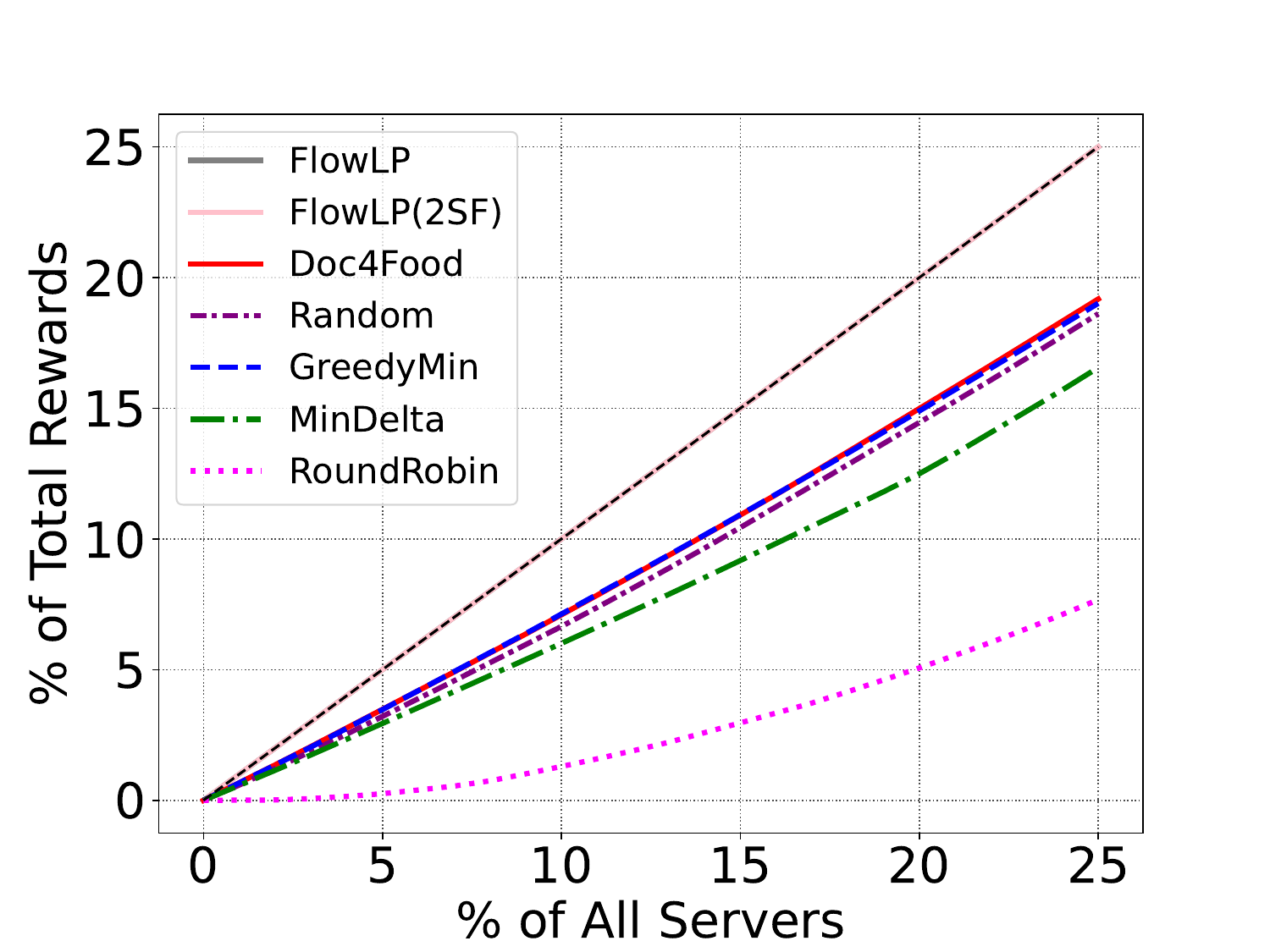} \label{fig:syn-dense}} 
  \hfill
  \subfigure[]{\includegraphics[width=0.32\textwidth]{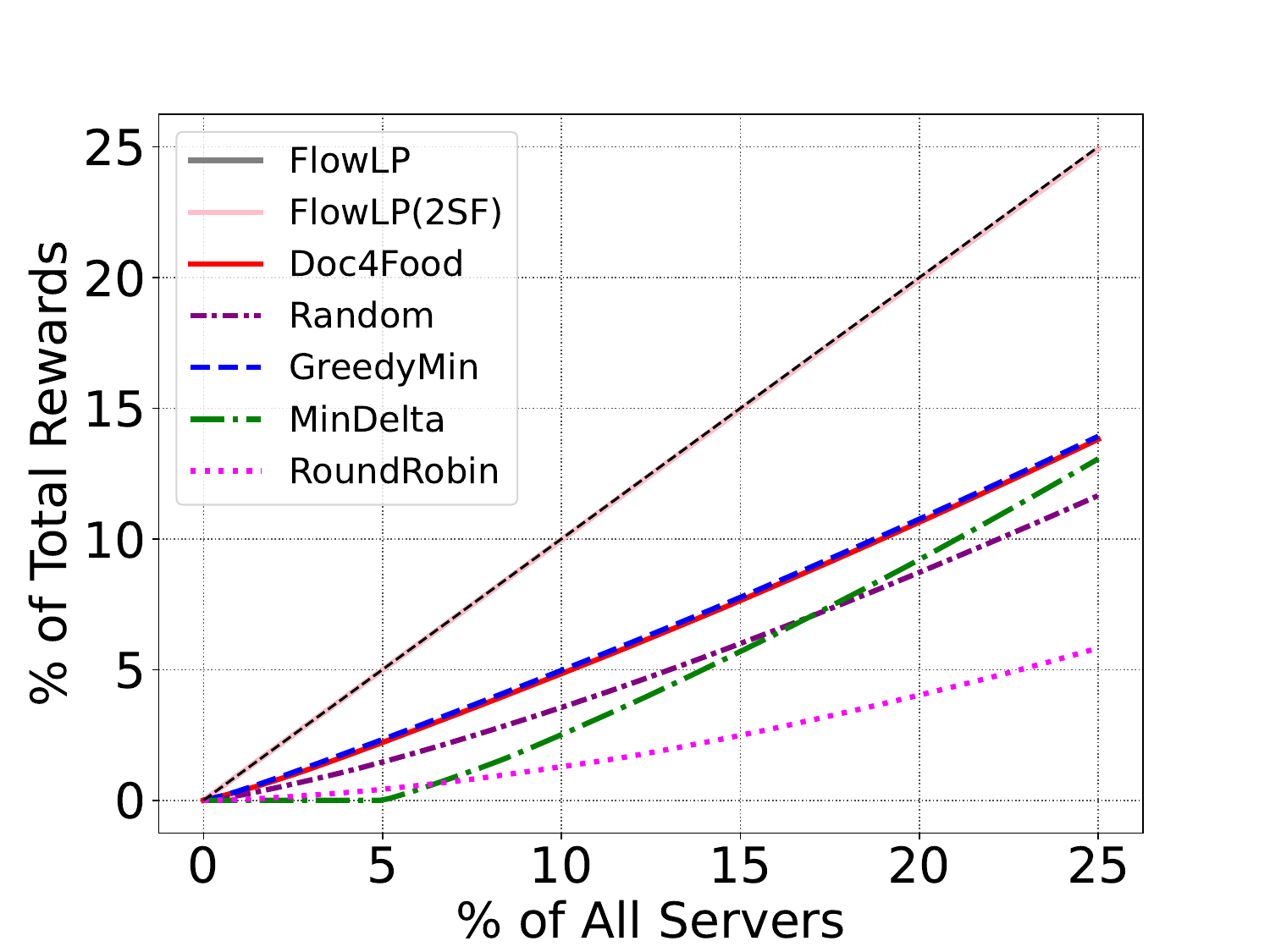} \label{fig:flowlp}} 
  \vspace*{-4mm}
  \caption{Lorenz curves for various algorithms on the \sparse\ (left), \dense\ (middle), and food-delivery datasets.}
\vspace*{-4mm}
\end{figure*}

\vspace{1mm} \noindent
\textbf{Cost Efficiency.} 
Focusing only on maximizing the minimum server reward might result in \flowlp\ deliberately placing  servers at locations away from future requests. 
In real world food delivery setting, such redundant increments in rewards can lead to an unnecessary rise in travel expenses 
i.e., cost to the platform. This scenario clearly implies an inherent cost-fairness trade-off. 
To consider both platform and server perspectives, we introduce an extra constraint in the \flowlp\ that upper bounds the total server rewards by a constant multiple of the cumulative edge costs associated with source-destination edges of each request. We choose these edges because they are invariant to the algorithm's other routing decisions. 
\begin{equation}
    \sum\limits_{i=1}^{k}m_i \leq \alpha.\sum\limits_{i=1}^{k} f_{(s_j, d_j)}^{i}, \quad \forall r_j \in \sigma \label{eqn:8}
\end{equation}
The additional constraint is represented as equation (\ref{eqn:8}), where $\alpha$ is a tunable parameter that controls the cost-fairness trade-off. The lower the value of $\alpha$, the higher the emphasis on reducing the platform-cost and vice-versa. Including this constraint makes \flowlp\ considerate towards both platform and the servers, we name the instance of \flowlp\ with this constraint as Two-Sided \flowlp\ (\flpsf).


\vspace{1mm} \noindent
\textbf{Online Algorithms.} 
\label{subsec:online-algos}
In this work, we consider the following online algorithms pertaining to the \textit{max-min} objective. The guiding principle underlying these approaches is to prioritize an \textit{eligible} server with the minimum accumulated reward while assigning servers to requests. An \textit{eligible} server, with respect to a request $r$, is one that is \textit{available} (not currently serving any other request) to serve the request as well as \textit{reachable} within $r$'s preparation time.
\begin{itemize}
    \item \textbf{\random.} The core idea 
    here is that as a server starts gaining rewards, it becomes exponentially harder for it to get further rewards. Specifically, an upcoming request $r$ at time $t$ is assigned to one of the eligible servers $i$ with a probability proportional to $2^{-x_i}$ where $x_i$ is the accumulated reward of the $i$-th server till time $t$. It is similar to increasing a server's weight as it accumulates more rewards that, in turn, make its movements harder. The time complexity per request is $\bigO(k)$.
    
    \item \textbf{\greedyMin.} In this online algorithm, an upcoming request $r$ is assigned to an eligible server with the minimum reward so far. This can be viewed as a specific instance of the \random\ approach, where the server with the minimum reward is assigned with a unit probability. The time complexity per request is $\bigO(k)$.
    
    \item \textbf{\docfood.} Drawing inspiration from the classical \dc\ algorithm~\cite{double-cover} commonly seen in the context of \kserver\ problems, we propose {\bf DOuble Coverage for FOOd Delivery (\docfood)} algorithm that combines \greedyMin\ with a heuristic informed by the domain knowledge specific to the food-delivery sector.
    
    In food delivery, incoming orders originate from a predetermined set of locations corresponding to various restaurants. This set of locations typically forms a small subset of all vertices in the metric space. With respect to \kfood, this translates into 
    a distinct subset of points, say $\rset$, within the metric space, representing the potential arrival locations for requests.
    
    Transitioning to the \docfood\ algorithm, upon the arrival of a request $r$, akin to the \greedyMin\ strategy, the eligible server with the minimum accumulated reward is chosen to move to $r$'s source location. While it moves towards $r$, all available servers also move virtually, by a small distance, towards their nearest nodes in $\rset$. Such server movements are also referred to as \textit{non-lazy} movements. This 
    mimics the actual practice of delivery drivers' movement to the nearest restaurant location (or market area) when they are idle~\cite{fairassign}. Note that while assigning servers, the virtual locations are considered for determining eligibility, but the actual locations are used for calculating rewards for the selected server. 
\end{itemize}


The consideration of \textit{fractional servers} alongwith the described \textit{edge pruning} allows us to efficiently solve \flowlp\ for many practical instances of the \fairkfood\ problem (refer Experimental Evaluation), leveraging advanced MILP solver~\cite{gurobi,cplex}. However, in contrast to \flowlp, the described online algorithms are rounded by default, 
i.e., they assign an entire server to a single request rather than using fractional assignments.  

\section{Experimental Evaluation}
\label{sec:experiments}
Next, we present a comprehensive experimental analysis on both synthetic and real-world food-delivery datasets.

\subsection{Experimental Framework.}
\label{subsec:eval-framework}
We conduct experiments on a machine with an Intel(R) Xeon(R) CPU @ $2.30$GHz and $252$GB RAM running on Ubuntu $20.04.5$ LTS. The entire codebase is written in Python $3.9$, and the Gurobi optimizer~\cite{gurobi} is used for solving the linear programs. 

\subsection{Baselines.} 
\label{subsec:baselines}
\begin{itemize}
    \item \flowlp: The fractional offline optimal algorithm for the \fairkfood\ problem.
    \item \random\ and \greedyMin: As explained in the subsection \textit{Online Algorithms} under \textit{Methodology}.
    \item \minDelta: A purely online counterpart to the heuristic-based semi-online algorithm developed by~\cite{fairfoody} aiming to minimize the reward gap between the minimum and maximum earning servers. Takes $\bigO(k^2)$ time per request assignment.
    \item \roundRobin: 
    Here, an upcoming request is assigned to the first \textit{eligible} server in round-robin manner. The server assignment complexity per request is $\bigO(k)$. 
\end{itemize}

\subsection{Evaluation Metrics.}
\label{subsec:eval-metrics}
We consider the following evaluation metrics:
\begin{itemize}
    \item \textbf{Number of infeasible requests (\#Unserved).} The number of requests that the corresponding algorithm could not serve. This might happen due to a scarcity of \textit{eligible} (available and reachable) servers for the given request. 
    \item \textbf{Minimum Reward (Min.R).} Given that our objective is to maximize the minimum reward, we record the minimum reward among the server rewards. The higher the minimum reward, the fairer the algorithm. If it is $0$, we look at the number of servers with a $0$ reward. 
    \item \textbf{Cost.} We define the cost of an algorithm as the average of all the server rewards. Since, for a given request, the source-destination distance is fixed, this cost essentially represents the algorithm's routing and server-assignment decisions. In real-world applications such as food delivery, ride-sharing, etc., it represents the total cost incurred by an online platform to compensate its delivery drivers (also referred to as the \textit{platform-cost}). 
\end{itemize}

\subsection{Experiments with Synthetic Data}
\begin{table*}[t]
    \parbox{0.45\linewidth}{
      \centering
      \caption{Comparison of \flowlp\ against various online algorithms on the food-delivery dataset. The values in the parentheses indicate the number of servers with $0$ reward.}\label{tab:flowlp}
      \begin{tabular}{cccc}
        \toprule
         & \textbf{\#Unserved} & \textbf{Cost} & \textbf{Min.R.}\\
        \midrule 
        \flowlp & $7$ & $19049.20$ & $19049.20$ \\
        \flowlp\textsc{(2S)} & $8$ & $8505.19$ & $8505.19$ \\
        \midrule
        \random & $8$ & $7645.31$ & $786$ \\
        \greedyMin & $11$ & $7508.34$ & $2319$ \\
        \docfood & $7$ & $7574.29$ & $2381$ \\
        \minDelta & $0$ & $6714.84$ & $0 (32)$ \\
        \roundRobin & $25$ & $7405.97$ & $0 (1)$ \\
        \bottomrule
        \end{tabular}  
    }
    \hfill
    \parbox{0.45\linewidth}{
    \centering
    \caption{Offline delivery algorithms vs Online algorithms on the food-delivery dataset. The values in the parentheses indicate the number of servers with $0$ reward.} \label{tab:fmff}
    \begin{tabular}{cccc}
    \toprule
     & \textbf{\#Unserved} & \textbf{Cost} & \textbf{Min.R}\\
    \midrule
    \foodmatch & $2$ & $8357.47$ & $0 (379)$ \\
    \fairfoody & $170$ & $9203.52$ & $0 (268)$ \\
    \workfood & $276$ & $8693.66$ & $0 (308)$ \\
    \midrule
    \random & $0$ & $15705.28$ & $0 (269)$ \\
    \greedyMin & $35$ & $8481.59$ & $0 (223)$ \\
    \docfood & $30$ & $8672.00$ & $0 (223)$ \\
    \minDelta & $0$ & $11589.69$ & $0 (265)$ \\
    \roundRobin & $56$ & $5431.93$ & $0 (263)$ \\
    \bottomrule
    \end{tabular}
    }
\vspace{-2mm}
\end{table*}

\label{subsec:synthetic-exp}

\subsubsection{Dataset and Setup.} We begin by establishing a graph $\xset$ composed of $500$ nodes. Using the Erdos-Renyi model, we add edges with an edge connection probability denoted as $p$, while ensuring $\xset$ remains connected. Subsequently, we generate two distinct datasets: one with $p=0.5$, referred to as \sparse, and another with $p=0.9$, known as \dense. This deliberate variation enhances our analysis by encompassing different network structures. The edge weights within $\xset$ are uniformly selected at random from a set of values $\{10, \ldots, 10000\}$ 
\footnote{Alternate weight selection methods, such as uniform or exponential, yield qualitatively similar results.}.

For each dataset, we generate a total of $250$ requests. Recall that each request $r$ in the \kfood\ problem is a $4$-tuple $(s, d, t^b, t^e)$. For each request, we sample $s$, $d$ from $\xset$ and $t^b$, $t^e$ from the interval $[100,900]$, such that the preparation-time ($t^e$-$t^b$) lies in the interval $[1,100]$. The delivery time of a request is set equal to the distance between the $s$ and $d$ (essentially, the speed of each server is assumed to be $1$ unit per timestep). Additionally, we maintain that $t^e$ for each request $r$ is distinct. Consequently, we have a set of $250$ requests that arrive at one of the $500$ nodes of $\xset$ over a span of $1000$ timesteps. 
The choice of the edge weights and the data configuration described above have been inspired by the characteristics of the real-world food-delivery dataset. We assume that all the servers are active for the entire duration of $1000$ timesteps. We use $\alpha=1.2$ for the \flpsf\ algorithm.\\

\begin{figure}
    \centering
    \includegraphics[width=0.35\textwidth]{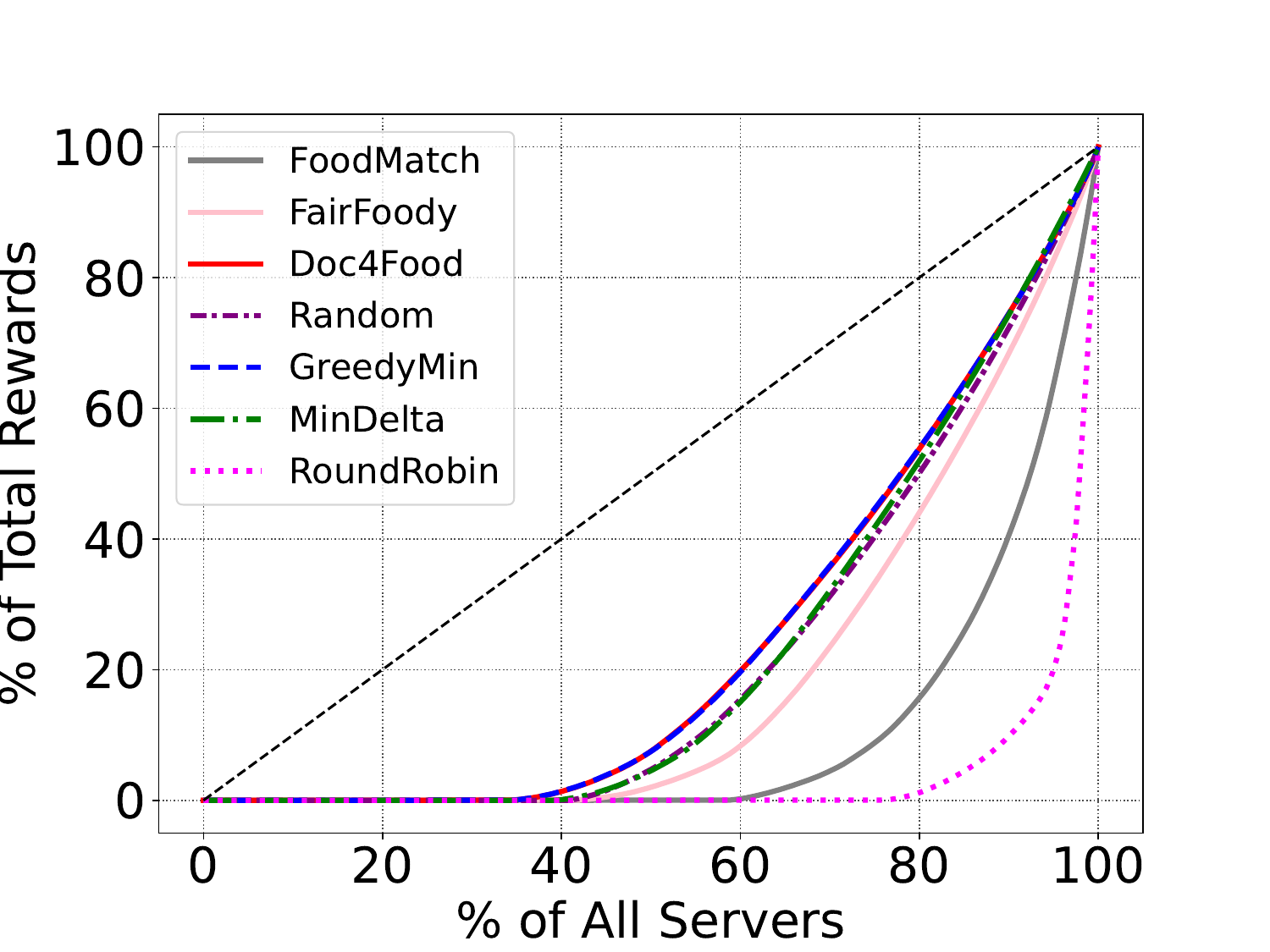}
    \vspace{-2mm}
    \caption{
    Offline delivery algorithms \foodmatch\ and \fairfoody\ vs. online algorithms on food delivery dataset.}
    \label{fig:fmff}
    \vspace{-4mm}
\end{figure}

\noindent
\textbf{Results.} Tables (\ref{tab:syn-sparse}) and (\ref{tab:syn-dense}) show results on the \sparse\ and \dense\ datasets respectively. We observe qualitatively similar results for both datasets. The fractional offline algorithms \flowlp\ and \flpsf\ perform optimally, achieving equal rewards for all servers. As intended, \flpsf\ reduces the cost while maintaining similar reward distribution compared \flowlp. Among online algorithms, \docfood\ achieves the highest min. reward while serving a nearly maximal number of requests while \minDelta, on the other hand, performs poorly in terms of both feasibility and minimum reward maximization. Notably, \docfood, as intended, reduces infeasibility as compared to \greedyMin\ due to its non-lazy server movements. \random\ and \roundRobin\ do well in terms of feasibility but fall short in increasing the minimum reward. Note that a higher (or lower) cost incurred by an algorithm can be primarily due to the more (or lesser) number of requests it serves. 

Figures (\ref{fig:syn-sparse}) and (\ref{fig:syn-dense}) depict the Lorenz curves corresponding to various algorithms, focusing on servers within the bottom $25$ percentile in terms of rewards. The closer a curve is to the line of equality, the more the fraction of net rewards captured by the corresponding fraction of the servers. The curves for \flowlp\ and \flpsf, obviously, intersect with the line of equality. Remarkably, we see that \docfood, closely followed by \greedyMin, outperforms all other online algorithms, raising the sum of rewards earned by the bottom $25\%$ earners to nearly $20\%$ of the net server rewards.

\subsection{Experiments with Real Food Delivery Data} 
\label{subsec:food-delivery}
\subsubsection{Dataset.}
We obtained a real-world Indian food delivery dataset by contacting the authors of~\cite{fairfoody}. It comprises $6$ days of food-delivery data from $3$ major Indian cities. 
The dataset spans Wednesday to Monday, effectively encompassing order load trends for both weekdays and weekends. Notably, the dataset provides thorough information about the delivery drivers (or servers) and orders (or requests) such as delivery vehicle trajectories, road networks of the cities, drivers' chosen working shifts (or active time durations), vehicle IDs, restaurant locations, customer locations, arrival-, pickup-, and delivery-time of each order, among other information. 

\subsubsection{Additional baselines.} In addition to the baseline algorithms described earlier, we consider two offline algorithms from the domain of online food delivery:
\begin{itemize}
    \item \foodmatch: An efficient, heuristic-based last-mile delivery algorithm introduced by~\cite{foodmatch} broadly based on the ideas of order-batching and order-driver bipartite matching~\cite{foodmatch}. It can be viewed as the semi-online counterpart to \greedyMin. 
    \item \fairfoody: A fair food delivery algorithm that tries to achieve an equitable driver income distribution via minimizing the income gap between the minimum and maximum earning drivers~\cite{fairfoody}.
    \item \workfood: A fair food delivery algorithm that attempts to provide minimum wage guarantees to the gig delivery drivers by balancing the demand and supply of the drivers in the platform.
\end{itemize}

\subsubsection{Setup.} We present an experimental evaluation on a subset of the dataset comprising the first $8$ hours of data of one of the three cities, averaged across all days. It consists of $1034$ orders and $650$ drivers. We selected this particular subset because it corresponds to the largest instance of the \flowlp\ problem solvable under the computational constraints of our evaluation framework. \footnote{Scaling to the entire 24-hour duration is feasible through practical techniques like re-optimization~\cite{reopt}. 
} The length of each timestep is $1$ second. Most orders in the dataset already have distinct arrival times, otherwise we ensure the same by shifting the orders in time up to a few minutes. 

The dataset includes information about specific pre-defined intervals in the day when drivers choose to be active, referred to as ``work-shifts''. For comparison against \flowlp, which assumes the servers to be active for all timesteps, we disregard the work-shifts and consider all the drivers to be active for the entire duration of $8$ hours. However, while comparing against more practical algorithms like \foodmatch\ and \fairfoody, which do consider the work-shifts, we also account for the same. Due to this distinction, we present the comparison between these algorithms and the online algorithms separately. 

For the \random\ baseline, presented evaluations are an average of $5$ runs of the algorithms. For the \flpsf\ algorithm, an $\alpha$ value of $5$ is chosen.


\subsubsection{Results.} Table (\ref{tab:flowlp}) shows the evaluations of the fractional offline solutions \flowlp\ and \flpsf\ and the online algorithms. The offline methods achieve optimal solutions where all servers attain equal rewards. \flpsf\ outperforms \flowlp\ in terms of cost efficiency. Among the online algorithms, \docfood\ performs the best in increasing the min. reward, slightly outperforming \greedyMin\, while being more feasible. The \random\ algorithm shows mediocre performance across all metrics. While \minDelta\ excels in minimizing costs, it fares poorly in terms of minimum reward, resulting in nearly $30\%$ servers receiving no rewards.

Table (\ref{tab:fmff}) presents the evaluations for the offline algorithms \foodmatch\ and \fairfoody\ along with the online baselines. The offline algorithms aren't designed to reject requests beyond preparation time so their corresponding infeasibility values denote pick-up deadline violations; no request was actually infeasible.  Again, \docfood\ leads to best min. reward slightly outperforming \greedyMin\, that too with fewer infeasbilities. It's worth noting that the online approach \minDelta\ demonstrates a close performance to \fairfoody, its offline counterpart. Other online algorithms exhibit similar trends as observed in table (\ref{tab:flowlp}).

Figures (\ref{fig:flowlp}) and (\ref{fig:fmff}) clearly show the superiority of \docfood\ in effectively elevating the net rewards of the bottom earners as compared to all other online algorithms.

\section{Conclusion}
In this work, we introduce two generalizations of the classical \kserver\ problem -- \kfood\ and \fairkfood -- with the ability to model a variety of real-world systems, where the \fairkfood\ problem assumes a \textit{maxmin} objective inspired by Rawls's theory of justice. We establish that these problems are strongly NP-hard and develop a versatile fractional offline optimal solution \flowlp\ for the \fairkfood\ problem. Moreover, we explore a strategy to achieve a better cost-fairness trade-off in \flowlp, leading to the \flpsf\ algorithm. We propose \docfood, a heuristic-based online algorithm for the food delivery domain. We conduct extensive experimentation on a synthetic dataset as well as a real-world food delivery dataset, using various offline and online algorithms. We hope our work can serve as a foundation for various interesting research problems related to enhancing online algorithms using machine-learned predictions, fairness in online algorithms, and deep learning for constrained optimization, among others. 

\vspace{1mm} \noindent
\textbf{Reproducibility.} Our codebase is available at {\url{https://github.com/ddsb01/Fair-kFood}.



\bibliography{main}
\end{document}